\documentclass[11pt,a4paper]{article}
\usepackage[utf8]{inputenc}
\usepackage[T1]{fontenc}
\usepackage{lmodern}
\usepackage{graphicx}
\usepackage{float} % Support the [H] placement option used by the four figures
\usepackage{epstopdf}
\usepackage{amsmath,amssymb,amsfonts,amsthm}
\usepackage{bm}
\usepackage{physics} % original \norm / \pdv macros
\usepackage{multirow}
\usepackage[normalem]{ulem}
\usepackage{cite}
\usepackage[margin=1in]{geometry}
\usepackage{xcolor}
\usepackage{hyperref}
\hypersetup{colorlinks=true,linkcolor=blue,citecolor=blue,urlcolor=blue}
\newtheorem{proposition}{Proposition}
\newtheorem{assumption}{Assumption}
\newtheorem{definition}{Definition}
\def\1{\mbox{\boldmath$1$}}

\def\B{\mbox{\boldmath$B$}}
\def\C{\mbox{\boldmath$C$}}
\def\D{\mbox{\boldmath$D$}}

\def\F{\mbox{\boldmath$F$}}

\def\H{\mbox{\boldmath$H$}}
\def\L{\mbox{\boldmath$L$}}
\def\I{\mbox{\boldmath$I$}}

\def\M{\mbox{\boldmath$M$}}
\def\N{\mbox{\boldmath$N$}}

\def\P{\mbox{\boldmath$P$}}
\def\Q{\mbox{\boldmath$Q$}}
\def\R{\mbox{\boldmath$R$}}
\def\S{\mbox{\boldmath$S$}}

\def\W{\mbox{\boldmath$W$}}
\def\X{\mbox{\boldmath$X$}}
\def\Y{\mbox{\boldmath$Y$}}

\newcommand{\LambdaB}{\bm{\Lambda}}

\def\b{\mbox{\boldmath$b$}}
\def\c{\mbox{\boldmath$c$}}

\def\f{\mbox{\boldmath$f$}}
\def\g{\mbox{\boldmath$g$}}
\def\h{\mbox{\boldmath$h$}}

\def\q{\mbox{\boldmath$q$}}
\def\r{\mbox{\boldmath$r$}}

\def\u{\mbox{\boldmath$u$}}
\def\v{\mbox{\boldmath$v$}}

\def\x{\mbox{\boldmath$x$}}
\def\y{\mbox{\boldmath$y$}}

\def\u{\mbox{\boldmath$u$}}
\def\taub{\mbox{\boldmath$\tau$}}
\def\0{\mbox{\boldmath$0$}}

\def\Nu{\bm{\nu}}

\newcommand*\proofdeltaVtitle{PROOF OF EQ.~\eqref{eq:deltaV}}
\newcommand*\proofSlopeTitle{PROOF OF EQ.~\eqref{eq:slope0}}

\title{Data-driven discrete-time deep neural network-based modeling for dissipative systems}
\author{Tuan Luong\thanks{Department of Mechanical Engineering, Sungkyunkwan University, Suwon 16419, South Korea. E-mail: \texttt{luongtuan@g.skku.edu}.}
\and Hyungpil Moon\thanks{Department of Mechanical Engineering, Sungkyunkwan University, Suwon 16419, South Korea. E-mail: \texttt{hyungpil@g.skku.edu}. Corresponding author.}}
\date{}

\begin{document}
\maketitle

\begin{abstract}
Physical AI has gained increasing attention for its role in developing AI systems that better understand, predict, and control real-world dynamics. Achieving this requires AI models that not only achieve high prediction accuracy but also preserve fundamental physical properties of dynamical systems. In this paper, we propose a deep discrete-time dissipative recurrent neural network (DissipNet) that explicitly enforces dissipativity, a key property related to stability and energy dissipation, through structural weight constraints and a dedicated training algorithm. By construction, the proposed network is capable of learning dissipative dynamics while preserving their inherent stability, which is formally analyzed using Lyapunov theory. In contrast to Physics-Informed Neural Networks (PINNs), which incorporate governing equations into the training loss but do not guarantee preservation of internal analytical properties such as dissipativity or passivity, our approach provides explicit guarantees on stability at the model level. We demonstrate the effectiveness of the proposed method through several modeling applications, and compare its performance with a naive recurrent neural network (RNN) and a PINN-based model.
\end{abstract}


\section{Introduction}
Neural Networks (NN) have spread as a powerful tool for multiple purposes in many fields such as signal processing, modelling, and control of dynamical systems. Many types of NNs, Multi-Layer Perceptron (MLP), Vanilla Recurrent Neural Networks (RNNs), Long-Short Term Memory (LSTMs), Convolutional Neural Networks (CNNs), Generative Models, etc., and their hybrid models~\cite{hagan}, have been proposed for different applications. However, in most of the work, NNs are considered only as a nonlinear system that fit a collection of data, which have no information about the physical characteristics. In general, a neural network does not preserve the physical properties of a real system. For example, the modeling of a dynamic system using a traditional data-based network model might result in a model that violates physical constraints~\cite{luong2021long}. Preserving physical properties of network models, such as stability properties, is therefore an important research topic to be investigated.

Amongst physical properties, dissipativity is a relevant issue that has been widely used to analyse the stability of linear and nonlinear systems. The dissipativity theory has found successful applications in various areas such as signal processing~\cite{xie1998passivity}, human-robot interaction~\cite{landi2018passivity}, impedance control~\cite{albu2007unified}, disturbance estimation and control of aircraft~\cite{slightam2021passivity}, mobile robot~\cite{ren2016passivity}. It may deal with nonlinear systems using only the general characteristics of the input–output dynamics and offers elegant solutions for the proof of absolute stability. Dissipative systems exhibit many desirable properties, for example, 1) the free dynamics and zero dynamics of Dissipative systems are Lyapunov stable; 2) the parallel and negative feedback interconnections of passive systems remain dissipative. This property allows the independent modelling and control of separate dissipative systems; and 3) a dissipative system can achieve stability using output feedback~\cite{liao2003global}. The property can be applied to design a controller for the systems that are unstable. In fact, there are many engineering systems (e.g., high-performance aircraft or even some more exotic circuits) which are even designed to be unstable for some range of initial conditions in order to take advantages of improved performance for other ranges of initial data. For those systems, an input feedforward dissipative (IFP) or an output feedback passivity (OFP) can be used to render the process dissipative~\cite{khalil2002nonlinear}.

Due to the importance of dissipative properties, effort has been made to study dissipavivity and passivity (related property) of neural networks. The passivity properties of static multilayer neural networks are examined in~\cite{commuri1997cmac}. By means of analysing the interconnection of error models, the study derived the relationship between passivity and closed-loop stability. Passivity properties of dynamic neural networks can be found in~\cite{yu2001some}. In~\cite{yu2003passivity}, dynamic multilayer neural networks are used for nonlinear system on-line identification. The passivity approach is applied to access several stability properties of the neuro identifier. The passivity conditions for delayed neural networks (DNNs) are considered in~\cite{li2005passivity}, where the passivity conditions for DNNs without uncertainties were derived, and then extend the results to the case with time-varying parametric uncertainties. The passivity conditions for DNNs with uncertainties are considered in~\cite{park2007further}. In~\cite{liao2003global}, the global dissipativity of a general class of continuous-time recurrent neural networks are addressed. It is worth pointing out that although some passivity analysis has been done, they require the solving of linear matrix equations which face challenges when scaling to large networks. Some recent studies on learning dissipative neural dynamics have been restricted to specific port-Hamiltonian network structures~\cite{drgovna2022dissipative}. However, these models are only applicable when the system inputs remain constant. In~\cite{xu2023learning}, the incremental dissipativity of a system with external inputs is considered. However, the approach does not consider the state model and requires multi-step learning process. In~\cite{okamoto2025learning}, a continuous dissipative Neural ODE model was obtained by projecting the original ODE onto a space that guarantees dissipativity. However, while the training process was time-consuming, a degradation in modeling accuracy may have occurred during the projection step.

In this paper, we aim to develop a discrete-time dissipative recurrent neural network (DissipNet) for learning a discrete model of a nonlinear system that is known to be dissipative. Our contribution includes the proposition of a linear matrix inequality condition for the network to ensure its dissipative property. We also contribute a sufficient learning algorithm that uses simple unconstrained optimization to determine weights that satisfy these conditions, which is simpler and provide more flexibility compared with ~\cite{xu2023learning}. The network’s stability is proven using Lyapunov theory. The integration of physical meaning into a generic network aligns with an emerging research topic known as physics-informed neural network (PINN)~\cite{cuomo2022scientific}. The underlying approaches are that PINN encode model equations, like Partial Differential Equations (PDE), as a component of the neural network itself. In this approach, NN must fit observed data while reducing a minimize a loss function, which includes reflecting the initial and boundary conditions along the space-time domain’s boundary and the PDE residual. It has been demonstrated that the PINN can efficiently leverage the predictive ability of neural networks PINNs addressing problems that are described by few data, or with noisy experiment observations. The work aims to learn the equations and replace them by NNs. However, there is no guarantee that the NN will preserve the equation’s internal analytical properties, for example, passivity in case of robotics, as opposed to our proposal.

The remaining of this paper is organized as follows. 
Section~\ref{sec:nnmodel} presents the dissipative neural network model and its learning algorithm. 
Case studies demonstrating the performance of the proposed network and comparisons with a PINN model and a Naive RNN model are presented in Section~\ref{sec:casestudy}. 
We conclude the paper in Section~\ref{sec:conclusion}.

\section{Dissipative Recurrent Neural Networks} \label{sec:nnmodel}
In this section, the problem under research will be described, and  the proposed network and the conditions for it to be dissipative will be given. Then, an algorithm on how the neural network will be trained will be presented.
\subsection{Problem formulation}
In this work, we consider the discrete nonlinear systems of the form
\begin{equation} \label{eq:ss1}
    \x_{k+1} = \f(\x_{k}, \u_{k})
\end{equation}
\begin{equation} \label{eq:ss2}
    \y_{k} = \g(\x_{k}, \u_{k})
\end{equation}
where $\x_{k} \in \mathbb{R}^{n \times n} , \u_{k} \in \mathbb{R}^{m \times 1}, \y_{k} \in \mathbb{R}^{p \times 1}$ are states, inputs and outputs of the system at time step $k$. $\f(0) = 0$, $\g(0) = 0 $.
\begin{definition}
    Definition 1: The system in Eq.(~\ref{eq:ss1}) and Eq.(~\ref{eq:ss2}) is said to be $(\Q, \S, \R)$ dissipative, where $0 \succeq  \Q \in \mathbb{R}^{p \times p}$, $\S \in \mathbb{R}^{m \times p}$ and $\R = \R^T \in \mathbb{R}^{m \times m}$  if there exists a Lyapunov function $V_{k}(\x_{k}) \geq 0$ such that
\begin{equation} \label{eq:qsrCond}
    V_{k+1}(\x_{k+1})-V_{k}(\x_{k}) \leq 
     [\y_k^T \quad \u_k^T]
    \begin{bmatrix}
        \Q & \S^T \\
        \S    & \R
    \end{bmatrix}
    \begin{bmatrix}
        \y_k\\
        \u_k 
    \end{bmatrix}
\end{equation}
\end{definition}
Some important cases of the $(\Q, \S, \R)$ dissipative are found in passivity, finite-gain $\mathcal{L}_2$ stability. A passive system is dissipative with supply rate given by $\Q = \0, \R = \0$, and $\S =  \frac{1}{2} \I$.  An $\mathcal{L}_2$ system is dissipative with supply rate $\S = \0, \Q = \frac{1}{\xi} \I$, where $\xi$ is the $\mathcal{L}_2$ gain of the system.
When the input-output data $(\u_k,\y_k)$ are given and are modeled using RNN networks, the system dynamics is represented by the following equations:
\begin{equation} \label{eq:rnn}
    \h_{k+1} = f_{rnn}(\h_{k},\u_k, \theta_{rnn})
\end{equation}
\begin{equation} \label{eq:output}
    \hat{\y}_k = \C \h^{L}_{k} + \D \u_k + \b_y
\end{equation}
\begin{equation} \label{eq:nn1}
    \min_{\substack{\theta_{rnn}, \C, \D}} \quad \sum_{i=1}^{N} \mathcal{L}(\hat{\y}_i, \y_i)
\end{equation}
where $\theta_{rnn}$ are trainable weights of the RNN network with L layers, $\h_{k} = [{\h^{0}_{k}}^T \quad {\h^{1}_{k}}^T \quad  \cdots \quad {\h^{L}_{k}}^T]^T $, $\h^{i}_k \in \mathbb{R}^{q \times 1} \quad (i= 1, \cdots, L)$ are the vectors of hidden states, $\C \in \mathbb{R}^{p \times q}, \D \in \mathbb{R}^{p \times m}$, and $\hat{\y}_i, \y_i$ are the output of the model and the ground truth, respectively at the time step $i (i=1, \cdots, N)$, where N is the number of data samples. If we choose $\x_k = \h_k$, the Eqs.~(\ref{eq:rnn}-\ref{eq:output}) have the same form as Eqs.~(\ref{eq:ss1}-\ref{eq:ss2}). Note that for a discrete-time system to be dissipative, $D \neq \0$ is necessary~\cite{byrnes1994losslessness}.

\subsection{Dissipative Recurrent Neural Networks}

The model in Eq.~(\ref{eq:rnn}) can be expanded in the following form
\begin{equation}\label{eq:newnn1}
\begin{aligned}
\h^{1}_{k+1} &= \phi(\W^{1}_{ih} \u_k + \W^{1}_{hh} \h^{1}_{k} + \b^{1}_{h}) \\
\h^{2}_{k+1} &= \phi(\W^{2}_{ih} \h^{1}_{k+1}+\W^{2}_{hh} \h^{2}_{k} + \b^{2}_{h}) \\
&\vdots \\
\h^{L}_{k+1} &= \phi(\W^{L}_{ih} \h^{L-1}_{k+1}+\W^{L}_{hh} \h^{L}_{k}+ \b^{L}_{h})
\end{aligned}
\end{equation}
\begin{equation}
    \hat{\y}_k = \C \h^{L}_{k} + \D \u_k + \b_y
\end{equation}
where $\W^{1}_{ih} \in \mathbb{R}^{q \times m}$, ${\b^{i}_{h}}^T \in \mathbb{R}^{q \times q}, \W^{i}_{hh} \in \mathbb{R}^{q \times q} \quad (i=1, \cdots, L)$, \quad $\W^{i}_{ih} \in \mathbb{R}^{q \times q} \quad (i=2, \cdots, L)$ and $\B_h \in \mathbb{R}^{n \times q}$,$\W^{1}_{ih}  \in \mathbb{R}^{q \times m}$. $\phi(\cdot)$ is the activation function. 

Eq.~(\ref{eq:newnn1}) can be represented as follows
\begin{equation} 
    \h_{k+1} = \phi(\W_{u} \u_k + \W_{ih} \h_{k+1} + \W_{hh}\h_{k}+\b_{h})
\end{equation}
\begin{equation}
    \hat{\y}_k = \C_h \h_{k} + \D \u_k + \b_y
\end{equation}
Where 
\begin{equation} \label{eq:Whh}
\W_{hh} =    
\begin{bmatrix}
\W^{1}_{hh} & \0 & \cdots & \0 \\
\0 & \W^{2}_{hh} & \cdots & \0 \\
\vdots & \vdots & \ddots & \vdots \\
\0 & \0 & \cdots & \W^{L}_{hh} \\
\end{bmatrix} 
\in \mathbb{R}^{L q \times L q}
\end{equation}
%
%%%
\begin{equation} \label{eq:Wih}
\W_{ih} =    
\begin{bmatrix}
\0 & \0 & \cdots & \0 \\
\W^{2}_{ih} & \0 & \cdots & \0 \\
\vdots & \vdots & \ddots & \vdots \\
\0 & \0 & \W^{L}_{ih} & \0 \\
\end{bmatrix} 
\in \mathbb{R}^{L q \times L q}
\end{equation}
\begin{equation}
\W_u =    
\begin{bmatrix}
\W^{1}_{ih}  \\
\0 \\
\0 \\
\vdots \\
\0 \\
\end{bmatrix} 
\in \mathbb{R}^{Lq \times m}
\end{equation}
\begin{equation}
\C_h =    
\begin{bmatrix}
\C & \0 & \0 \cdots & \0
\end{bmatrix} \in \mathbb{R}^{p \times Lq}
\end{equation}
\begin{equation}
    \b_{h} = [{\b^{1}_{h}}^T \quad {\b^{2}_{h}}^T \quad  \cdots \quad {\b^{L}_{h}}^T]^T 
    \in \mathbb{R}^{Lq \times 1}
\end{equation}
\begin{equation}
    \b_{y} \in \mathbb{R}^{p \times 1}
\end{equation}
For a sequence length of T, the RNN network can be represented as
\begin{equation} \label{eq:ss1_2}
\begin{aligned}
\h_{k+1} &= \phi(\W_{u} \u_k + \W_{ih} \h_{k+1} + \W_{hh}\h_{k}+\b_{h}) \\
\h_{k+2} &= \phi(\W_{u} \u_{k+1} + \W_{ih} \h_{k+2} + \W_{hh}\h_{k+1}+\b_{h})\\
&\vdots \\
\h_{k+T} &= \phi(\W_{u} \u_{k+T} + \W_{ih} \h_{k+T} + \W_{hh}\h_{k+T-1}+\b_{h})
\end{aligned}
\end{equation}
Then
\begin{equation} \label{eq:ss2_2}
    \hat{\y}_{k+T-1} = \C_h \h^{L}_{k+T-1} + \D \u_{k+T-1} +\b_{y}
\end{equation}
Eqs.~(\ref{eq:ss1_2})-(\ref{eq:ss2_2}) can be converted into the form of Eqs.~(\ref{eq:ss1})-(\ref{eq:ss2}) as follows
\begin{equation}\label{eq:ss1_3}
\begin{aligned}
\Nu_{k+1}&= \phi_1(\W_{u} \u_k + \W_{ih} \Nu_{k+1} + \W_{hh}\Nu_{k}) \\
\Nu_{k+2}&= \phi_1(\W_{u} \u_{k+1} + \W_{ih} \Nu_{k+2} + \W_{hh}\Nu_{k+1}) \\
&\vdots \\
\Nu_{k+T} &= \phi_1(\W_{u} \u_{k+T} + \W_{ih} \Nu_{k+T} + \W_{hh}\Nu_{k+T-1})
\end{aligned}
\end{equation}
Then
\begin{equation} \label{eq:ss2_3}
    \hat{\y}_{k+T-1} = \C_h \Nu ^{L}_{k+T-1} + \D \u_{k+T-1} 
\end{equation}
%%%
where $\Nu_k = \h_k-\phi(\0), \phi_1 = \phi -\phi(\0),  \b_h = (\W_{ih}+\W_{hh})\phi(\0),$ $\b_y = \C_h\phi(\0),$
\begin{assumption} \label{assump1}
The activation function $\theta$ is piecewise differentiable and has slope restricted in $[\alpha = 0, \beta]$
    \begin{equation}
        \alpha = 0 \leq \frac{\phi(x_2)-\phi(x_1)}{x_2 - x_1} \leq \beta
    \end{equation}
\end{assumption}
This assumption holds for the most widely used activation functions, such as tanh, ReLU, and the sigmoid function. While the dissipativity condition proposed in this work can be easily extended for all finite values of $\alpha$ and $\beta$, the proposed training algorithm is restricted to the case $\alpha = 0$. The extension to cases with $\alpha \neq 0$ is left for future work.
\begin{proposition}
The systems described in Eqs.~(\ref{eq:ss1_3})-(\ref{eq:ss2_3}) will be $(\Q, \S, \R)$ dissipative if there is a positive definite matrix $\P = \P^ T$ such that
    \begin{equation} \label{eq:lmi}
        \begin{bmatrix}
        \P + \C_h^T \Q \C_h & -\gamma \W_{hh}^T\LambdaB  & \C_h^T \S^T + \C_h^T \Q \D \\
        -\gamma\LambdaB \W_{hh}    & \LambdaB  - \gamma \W^*_{ih} -\P & -\gamma \LambdaB  \W_u \\
        \S \C_h + \D^T \Q C_h & -\gamma \W_u^T\LambdaB  & \R_1
        \end{bmatrix} \\
    \end{equation}
is positive definite. 

where $\W^*_{ih} = \LambdaB \W_{ih} + \W_{ih}^T \LambdaB$, $\R_1 =  \R + \S \D + \D^T \S^T +  \D^T \Q \D$, $\LambdaB = diag(\lambda_1,  \lambda_2, \cdots, \lambda_{Lq})$ with $\lambda_i > 0 \quad (i=1, \cdots, Lq)$.
\end{proposition}
\begin{proof}
    Consider the Lyapunov function as follows
\begin{equation} \label{eq:Vkz}
    V_k(\Nu_k) = \Nu_{k}^T\P\Nu_{k}
\end{equation}
Left and right multiplying Eq.~(\ref{eq:lmi}) by $[\Nu_k^T \quad \Nu_{k+1}^T \quad \u_k^T]$ and $[\Nu_k^T \quad \Nu_{k+1}^T \quad \u_k^T]^T$, where $\v_k = \W_{u} \u_k + \W_{ih} \Nu_{k+1} + \W_{hh}\Nu_{k}$, respectively, one can obtain [see Appendices~\ref{sec:provedeltaV}]. 
\begin{multline} \label{eq:deltaV}
    V_{k+1}(\Nu_{k+1})-V_{k}(\Nu_{k}) \leq 
     [\y_k^T \quad \u_k^T]
    \begin{bmatrix}
        \Q & \S^T \\
        \S    & \R
    \end{bmatrix}
    \begin{bmatrix}
        \y_k\\
        \u_k 
    \end{bmatrix}\\
    + [\v_k^T \quad \Nu_{k+1}^T]
    \begin{bmatrix}
        \0 & -\gamma \LambdaB \\
         -\gamma \LambdaB    & \LambdaB
    \end{bmatrix}
    \begin{bmatrix}
        \v_k\\
        \Nu_{k+1} 
    \end{bmatrix} \\
    \leq 
     [\y_k^T \quad \u_k^T]
    \begin{bmatrix}
        \Q & \S^T \\
        \S    & \R
    \end{bmatrix}
    \begin{bmatrix}
        \y_k\\
        \u_k 
    \end{bmatrix}
\end{multline}
Moreover, it can be shown that [see Appendices~\ref{sec:slopeproof}]
\begin{multline} \label{eq:slope0}
    [\v_k^T \quad \Nu_{k+1}^T]
    \begin{bmatrix}
        \0 & -\gamma \LambdaB \\
         -\gamma \LambdaB    & \LambdaB
    \end{bmatrix}
    \begin{bmatrix}
        \v_k\\
        \Nu_{k+1} 
    \end{bmatrix} 
    \leq 0
\end{multline}
From Eq.~(\ref{eq:deltaV}) and Eq.~(\ref{eq:slope0}), one has
\begin{multline}
    V_{k+1}(\Nu_{k+1})-V_{k}(\Nu_{k}) \leq 
     [\y_k^T \quad \u_k^T]
    \begin{bmatrix}
        \Q & \S^T \\
        \S    & \R
    \end{bmatrix}
    \begin{bmatrix}
        \y_k\\
        \u_k 
    \end{bmatrix}
\end{multline}
Which implies the network model is ($\Q, \S, \R$) dissipative.
\end{proof}
\subsection{Training algorithm for DissipNet} 
\label{sec:model}
The algorithm to find the weights of the network to guarantee the constraints in Eq.~(\ref{eq:lmi1}).

Step 1: Choose $\D$ such that $L = \R + \S \D + \D^T \S^T + \D^T \Q \D \succ 0$ [See Appendices~\ref{sec:D}]

If $\R_1 = \R + \S \D + \D^T \S^T + \D^T \Q \D \succ 0$, using Schur complement the LMI condition in Eq.~(\ref{eq:lmi}) can be rewritten as follows 
\begin{multline} \label{eq:lmi1}
\begin{bmatrix}
\P + \C_h^T \Q \C_h & -\gamma \W_{hh}^T\LambdaB  \\
-\gamma \LambdaB \W_{hh}    & \LambdaB  - \gamma \W^*_{ih} -\P \\
\end{bmatrix} -\\
 \begin{bmatrix}
\C_h^T \S^T + \C_h^T \Q \D \\
 -\gamma \LambdaB \W_u 
\end{bmatrix}
\R_1^{-1}
\begin{bmatrix}
\C_h^T \S^T + \C_h^T \Q \D \\
 -\gamma \LambdaB \W_u 
\end{bmatrix}^T 
\succ 0
\end{multline}

Step 2. Choose $\C_h, \B, \LambdaB \W_u$ as free variables, we will choose $\P, \LambdaB$ such that
\begin{multline} \label{eq:lmi2}
    \begin{bmatrix}
        \P + \C_h^T \Q \C_h & -\gamma\W_{hh}^T \LambdaB \\
        -\gamma\LambdaB  \W_{hh}    & \LambdaB  -\gamma \W^*_{ih} -\P \\
    \end{bmatrix} \succ \X^T \X \\
    + \begin{bmatrix}
    \C_h^T \S^T + \C_h^T \Q \D  \\
     -\gamma \LambdaB  \W_u 
    \end{bmatrix}
    \R_1^{-1}
    [\S \C_h + \D^T\Q \C_h  \quad  -\gamma \W_u^T \LambdaB ]
\end{multline}
Where $\X$ is a free variable. Denote
\begin{multline}
    \H = \begin{bmatrix}
    \H_{11} & \H_{12} \\
    \H_{21} & \H_{22}
    \end{bmatrix} = \X \X^T  \\
    + \begin{bmatrix}
    \C_h^T \S^T + \C_h^T \Q \D \\
     -\gamma\LambdaB \W_u 
    \end{bmatrix}
    \R_1^{-1}
    [\S \C_h + \D^T\Q\C_h   \quad  -\gamma \W_u^T\LambdaB]
\end{multline}
We bring another matrix $\H*$ that satisfies
\begin{multline}\label{eq:lmi0}
\H^{*}=
\begin{bmatrix}
    \H_{11} + \epsilon_1\I& \H_{12}^* + \Y\\
    \H_{12}^{*T} + \Y^T &  \H_{22}+ \epsilon_1\I 
    \end{bmatrix}
    \succ
    \begin{bmatrix} 
    \H_{11} & \H_{12} \\
    \H_{12}^T &  \H_{22} 
    \end{bmatrix}
\end{multline} 
where $\H_{12}^* $ is the matrix that is obtained from $\H_{12}$ 
by zeroing out all non-diagonal block matrices while preserving the diagonal block matrices in the same form as matrix $\W_{hh}$ in Eq.~(\ref{eq:Whh}). $\Y \in \mathbb{R}^{L q \times L q}$ is the trainable matrix that has the same form as $\W_{hh}$. For Eq.~(\ref{eq:lmi0}) to be satisfied, $\epsilon_1$ will be chosen such that
\begin{equation}
    \epsilon_1 \geq \norm{\H_{12}^* + \Y-\H_{12}}
\end{equation}

Finally, $\P$ and $\gamma \W_{hh}^T\LambdaB$ can be chosen as follows
\begin{equation} \label{eq:P}
    \P = \H_{11}-\C_h^T \Q \C_h + \epsilon_1\I \succ 0
\end{equation}

\begin{equation} \label{eq:Wz_cond2}
    -\gamma \W_{hh}^T\LambdaB  = \H_{12}^* + \Y %+ diag(ReLU(\gammab)) + diag(eig(\H_{12} - \H_{12}^*))
\end{equation}
where $diag(\cdot)$ and $ReLU(\cdot)$ are diagonal operator and Rectified Linear Unit (ReLU) function, respectively. %Where
It is easily seen that the Eqs.~(\ref{eq:P}-\ref{eq:Wz_cond2}) guarantee that 
\begin{multline} \label{eq:lmi3}
\begin{bmatrix}
        \P + \C^T \Q \C & - \gamma\W_{hh}^T\LambdaB  \\
        -\gamma\LambdaB \W_{hh}    & \LambdaB  - \gamma \W^*_{ih} -\P \\
    \end{bmatrix} \\
    \succ 
    \begin{bmatrix}
    \H_{11} + \epsilon_1\I& \H_{12}^* \\
    \H_{12}^{*T} &  \H_{22}+ \epsilon_1\I 
    \end{bmatrix}
\end{multline}
By choosing $\gamma \LambdaB \W_{ih}$ freely and $\LambdaB$ as in Eq.~(\ref{eq:lambda_cond2})
% \begin{equation} \label{eq:lambda_cond2}
%     \LambdaB  = diag(eig(\H_{22}+\P + \gamma \W^*_{ih}+ \epsilon_1\I))+ \epsilon_2\I
% \end{equation}
\begin{equation} \label{eq:lambda_cond2}
    \LambdaB =
    \left( \norm{ 
    \H_{22}+\P+\gamma\W^*_{ih}
    +\epsilon_1\I }
    +\epsilon_2
    \right)\I
\end{equation}
where $\epsilon_2 > 0.$
We can also obtain that 
\begin{equation} \label{eq:lmi4}
    \begin{bmatrix}
        \P + \C_h^T \Q \C_h & - \gamma\W_{hh}^T\LambdaB  \\
        -\gamma\LambdaB \W_{hh}    & \LambdaB  - \gamma \W^*_{ih} -\P \\
    \end{bmatrix} 
    \succ
    \begin{bmatrix}
    \H_{11} & \H_{12} \\
    \H_{21} & \H_{22}
    \end{bmatrix}
\end{equation}
Using Eq.~(\ref{eq:lambda_cond2}), matrices $\W_{hh}, \W_{ih}, \W_{u}$ then can be calculated, given $\LambdaB \W_{hh}, \LambdaB \W_{ih}, \LambdaB \W_{u}$.
From Eq.~(\ref{eq:lmi3}) and Eq.~(\ref{eq:lmi4}), Eq.~(\ref{eq:lmi1}) is satisfied.
\section{Case studies} \label{sec:casestudy}
In this section, the proposed DissipNet will be evaluated through two robotic modeling examples: a mass-spring-damper model and a two degree-of-freedom planar manipulator, which were both proven to be dissipative.
%%%%
\subsection{Mass-spring-damper system}
The equation of the model of a mass-spring-damper (MSD) model can be expressed as
\begin{equation}
    m\ddot{x}(t)+b\dot{x}(t)+kx(t) = f(t)
\end{equation}
Here $m, b, k$ are mass, damping, and spring coefficients, respectively. $f$ is the input of the system, and $x$ is the position of the mass.
The energy relation of the MSD system is as follows
\begin{multline}\label{eq:msd}
    \int_0^T f(t)\dot{x} = \frac{1}{2}kx^2(T) -\frac{1}{2}kx^2(0)+ \\
    \frac{1}{2}m\dot{x}^2(T)-\frac{1}{2}m\dot{x}^2(0)+\int_0^T b\dot{x}(t)dt
\end{multline} 
Eq.~(\ref{eq:msd}) implies that the continous system is dissipative with the following storage function and the supply rate
\begin{equation} \label{eq:storage_msd}
    V(x) = \frac{1}{2}kx^2(t) +\frac{1}{2}m\dot{x}^2(t)
\end{equation}
\begin{equation} \label{eq:supply_msd}
    w(u,y) = \begin{bmatrix}
        \y^T f
    \end{bmatrix}
    \begin{bmatrix}
        0 & 0   & 0 \\
        0 & -b  &\frac{1}{2}\\
        0 & \frac{1}{2} & 0
    \end{bmatrix}
    \begin{bmatrix}
        \y^T f
    \end{bmatrix}^T
\end{equation}
where $\y = \begin{bmatrix}
        x & \dot{x}
    \end{bmatrix}^T $
    
Parameters of the MSD model are follows:  m = 1 (kg), k = $4.2N/m$, b = $2.5Ns/m$, and the sampling time was 0.05s. 
\subsection{Two DOF manipulator}
Modeling of a planar two link manipulator is used to evaluate the performance of the DissipNet. The equation of motion of the manipulator can be expressed as follows~\cite{van2010consistent}
\begin{equation} \label{eq:2dof}
    \M(\q)\ddot{\q} + \C(\q, \dot{\q})\dot{\q} + \F_{friction} = \taub
\end{equation}
where $\q = [q_1 \quad q_2]^T$ denotes the joint angles of the manipulator.
\begin{align}
    \M(\q) = 
\begin{bmatrix}
    m_{11} & m_{12}\\
    m_{21} & m_{22}
\end{bmatrix}\\
    \C(\q,\dot{\q}) = 
\begin{bmatrix}
    c_{11} & c_{12}\\
    c_{21} & c_{22}
\end{bmatrix}\\
\F_{friction} = 
\begin{bmatrix}
    b_1 \dot{q}_1 & b_2 \dot{q}_2
\end{bmatrix}^T\\
    \taub= 
\begin{bmatrix}
    \tau_1 & \tau_2
\end{bmatrix}^T
\end{align}
$\M(\q), \C(\q, \dot{\q}), \F_{friction}, \taub$ are the mass matrix, Corriolis and Centrifugal matrix, friction torques, and the vector of applied torques to the joints, respectively. 
$m_{11} = m_1\frac{l_1^2}{4} + m_2(l_1^2 + l_2^2/4 + l_1 l_2cos(q_2)) + I_{z1} + I_{z2}$
$m_{12} = m_2(\frac{1}{4} + \frac{l_1 l_2}{4}cos(q_2)) + I_{z2}$, $m_{21} = m_{12}, m_{22} = m_2 \frac{l_2^2}{4}+I_{z2}$
$c_{11} = -m_2l_1l_2\dot{q}_2sin(q_2), c_{12} = -\frac{1}{2}l_1 l_2 \dot{q}_2 sin(q_2)$
$c_{21} = -\frac{1}{4}m_2 l_1 l_2 \dot{q}_2 sin(q_2) + \frac{1}{2}m_2 l_1 l_2 \dot{q}_1 sin(q_2)$,
$\c_{22} = -\frac{1}{4}m_2 l_1 l_2 \dot{q}_1 sin(q_2)$
\begin{multline}
\tau_1 = \tau_{m1} + g \left[ l_1\left( \frac{m_1}{2} + m_2 \right)\cos q_1 + \frac{m_2 l_2}{2} \cos(q_1 + q_2) \right] \\
- p_x \left[ l_1 \sin q_1 + l_2 \sin(q_1 + q_2) \right] +\\
p_y \left[ l_1 \cos q_1 + l_2 \cos(q_1 + q_2) \right], \\
\tau_2 = \tau_{m2} + \frac{1}{2} m_2 g l_2 \cos(q_1 + q_2) \\
- p_x l_2 \sin(q_1 + q_2) + p_y l_2 \cos(q_1 + q_2).
\end{multline}
%%%%%%%%%%%%%
Here, $m_i (i=1,2)$ denotes the mass of $i_{th}$ link; $l_i$ is the length of $i_{th}$ link, and $I_i$ is the moment of inertia of the link about the axis passing through the center of the mass and perpendicular to the plane of motion. $p_x$ and $p_y$ are the x and y components of the external forces, which are considered to be zero in this paper. $\tau_{m1}, \tau_{m2}$ are active joint torques. The parameters of the 2-dof manipulator were $m_1 = m_2 = 1 (kg)$, $l_1 = l_2 = 1 (m)$, $I_{z1} = I_{z2} = 0.1kgm^2$, $b_1 = b_2 = 0.01 Nms$. The sampling time was 0.05s. It can be shown that the input-output pair of the continuous system $\f$ and $\dot{\q}$ is passive\cite{albu2007unified}. 
\section{Evaluation results}
The DissipNet will be used to model the input-output relationship of the models described above. Moreover, we will compare the performance of our model with an RNN model without dissipativity property (Naive RNN) and a physics-informed neural network model (PINN)~\cite{cuomo2022scientific}. 

Here, the idea of the PINN modelling approach is to integrate the system's known dynamics by adding the physical model into the cost with a regularization coefficient.  In particular, the loss for the MSD and 2-DOF manipulator will be as in Eq.~\ref{eq:pinn_cost_msd} and Eq.~\ref{eq:pinn_cost_2dof}, respectively.
\begin{equation} \label{eq:pinn_cost_msd}
   \mathcal{L} = \sum_{i=1}^{N} \left[ \mathcal{L}(\y_i, \y^d_i) + \lambda (m \ddot{x_i} + b\dot{x_i} + kx_i - f_i)^2 \right]
\end{equation}
\begin{equation} \label{eq:pinn_cost_2dof}
    \mathcal{L} = \sum_{i=1}^{N} \left[ \mathcal{L}(\y_i, \y^d_i) + 
   \lambda \r_i^T \r_i  \right]
\end{equation} 
where $\r_i = \M(\q_i)\ddot{\q_i} + \C(\q_i, \dot{\q_i})\dot{\q_i} + \F_{friction, i} - \taub_i$

%%%
The dissipavitivy of the DissipNet model is evaluated using Eq.~(\ref{eq:qsrCond}) where
$V =  \Nu_k\P\Nu_k$ as shown in Eq.~(\ref{eq:qsrCond}). For the PINN model (continous model), the condition to show the dissipativity is
\begin{equation} \label{eq:dis_cont}
    \int_0^{T_k} yudt \geq V(T_k) - V(0)
\end{equation}
All simulations were conducted on a desktop workstation with an Intel(R) Core(TM) i9-9900K CPU @ 3.6\,GHz, 128\,GB RAM, and an NVIDIA GeForce RTX 2080 Ti GPU. 
For both the 2-DOF manipulator and the MSD systems, the RNN and DRNN models share the same architecture with a hidden dimension of 64, and 2 recurrent layers. 
The PINN baseline is a 4-layer feed-forward network with a hidden dimension of 128. 
All models were trained using the MSE loss and the Adam optimizer with a learning rate of $10^{-3}$ for 5000 and 10000 epochs for the 2-DOF manipulator and MSD systems, respectively.

\begin{figure}[H]
    \centering
    \includegraphics[width=0.5\columnwidth]{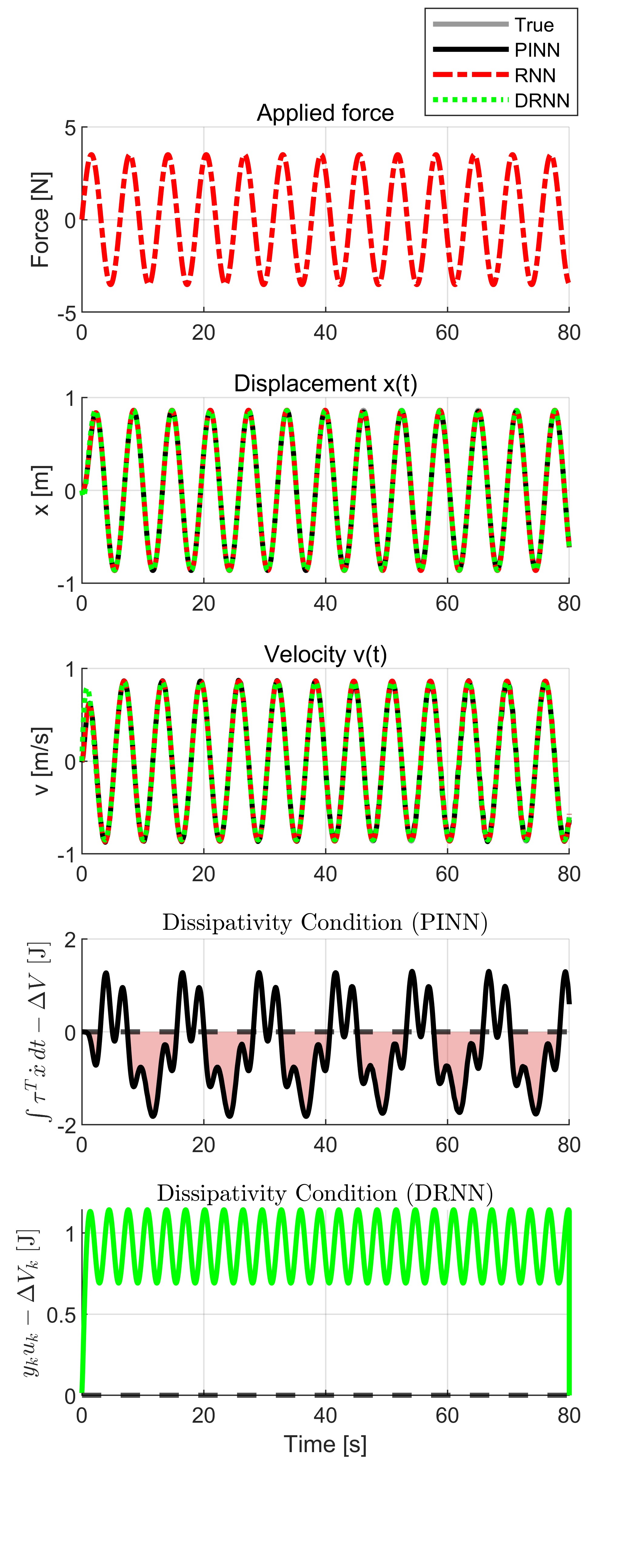}
    \caption{Comparison of DissipNet model (DRNN) and PINN model on a mass-spring-damper (MSD) model under an external input force. The second and third subplots from the top show that all RNN, DRNN and PINN models can capture the dynamics of the MSD system. However, the fourth subplot shows that the PINN model does not preserve dissipativity property, as opposed to the dissipativity property of the DNN model implied in the last subplot.  Here, the shaded color shows the area where the value is negative. Data from 0s to 40s were used for training, and data from 40s to 80s were used for testing.}
     \label{fig:drnn_pinn_msd}
\end{figure}
%%%
%
\begin{figure}[H] 
    \centering
    \includegraphics[width=0.5\columnwidth]{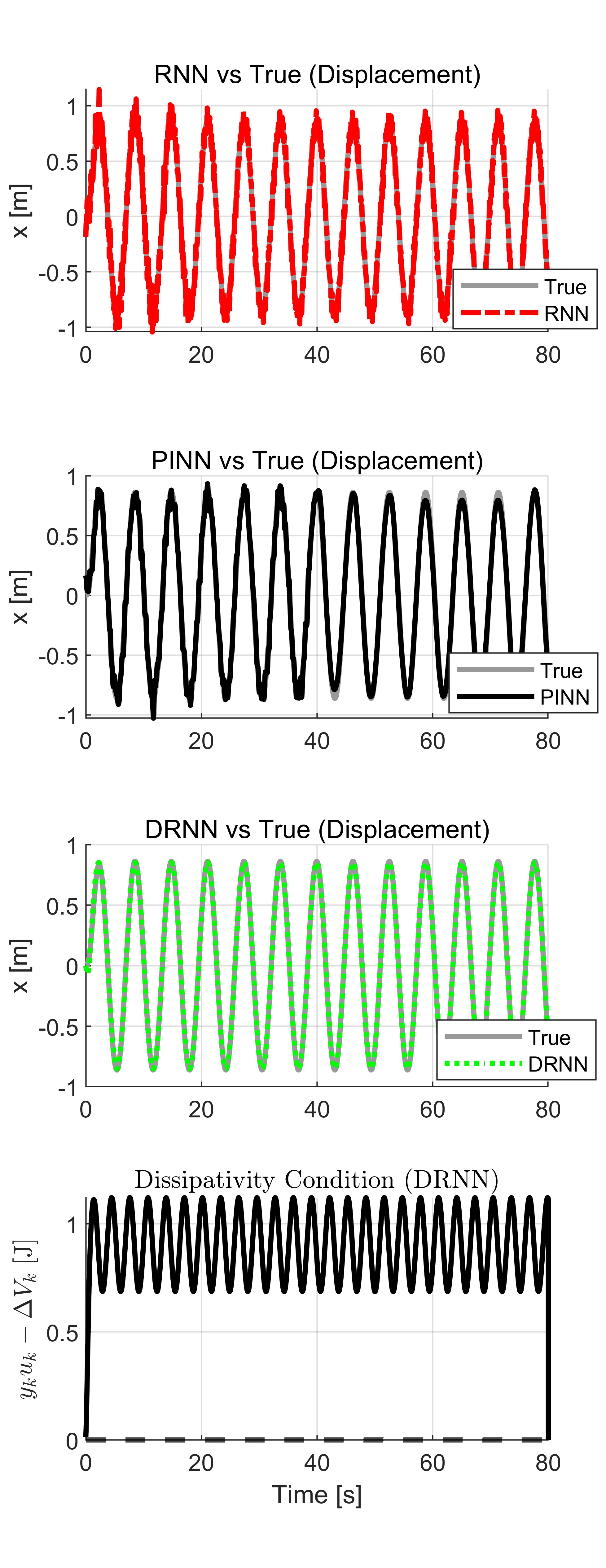}
    \caption{Comparison between RNN, DRNN and PINN models when the training data are subjected to noises. Here, Gaussian noise with zero mean and standard deviation $\sigma = 0.1$ is added to the outputs in the training. It is seen that the DRNN obtained better results compared with RNN and PINN while preserving dissipativity property, which indicates the robustness of the proposed DissipNet modelled when modeling dissipative systems. }
    \label{fig:drnn_rnn_msd_noise}
\end{figure}
\begin{figure}[H] 
    \centering
    \includegraphics[width=0.5\columnwidth]{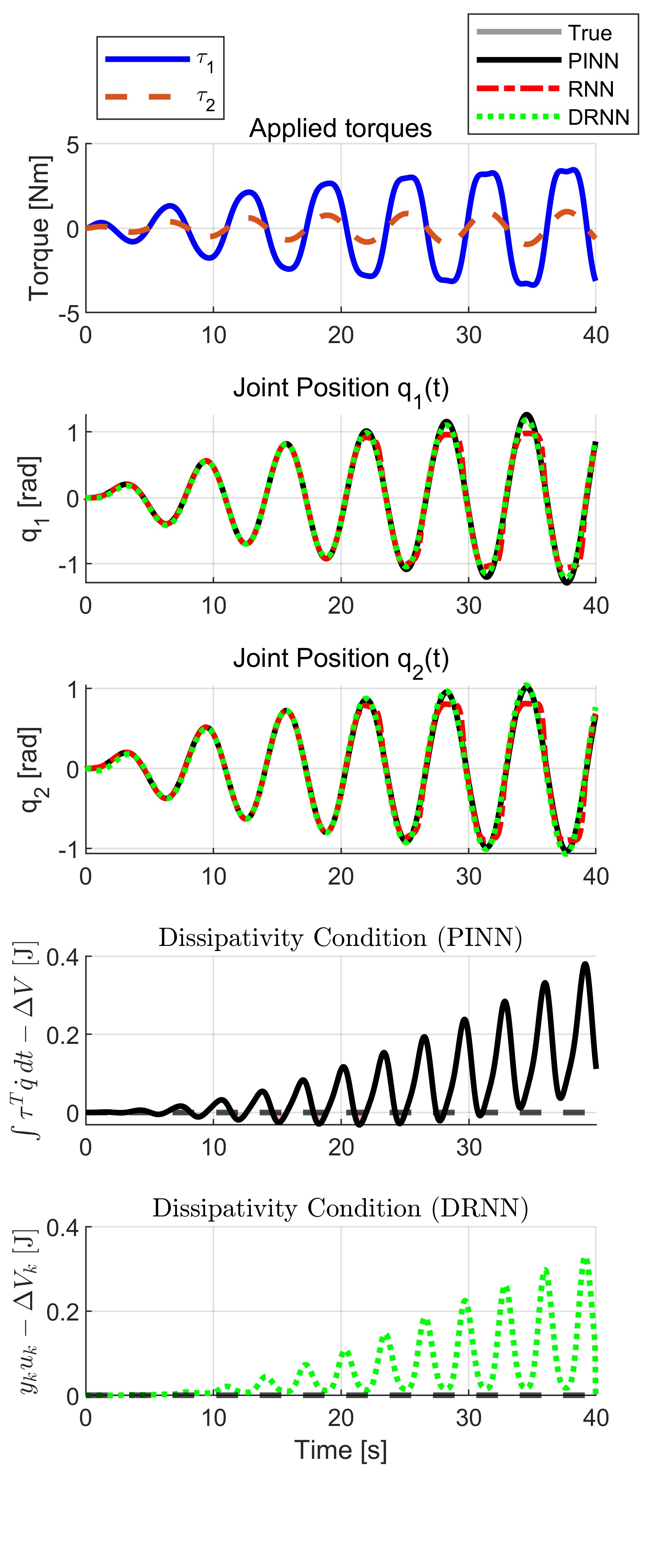}
    \caption{Comparison of the DissipNet model (DRNN), the Naive RNN model (RNN) and PINN model on the 2-dof manipulator modeling. The second and third subplots from the top show that all the naive RNN model, DRNN and PINN model can capture the dynamics of the system, however, the fourth subplot indicates that the PINN model is not dissipative as opposed to the DRNN model implied from the bottom subplot. Here, data from 0s to 20s were used for training, and data from 20s to 40s were used for testing.}
    \label{fig:drnn_rnn_2dof}
\end{figure}
\begin{figure}[H] 
    \centering
    \includegraphics[width=0.5\columnwidth]{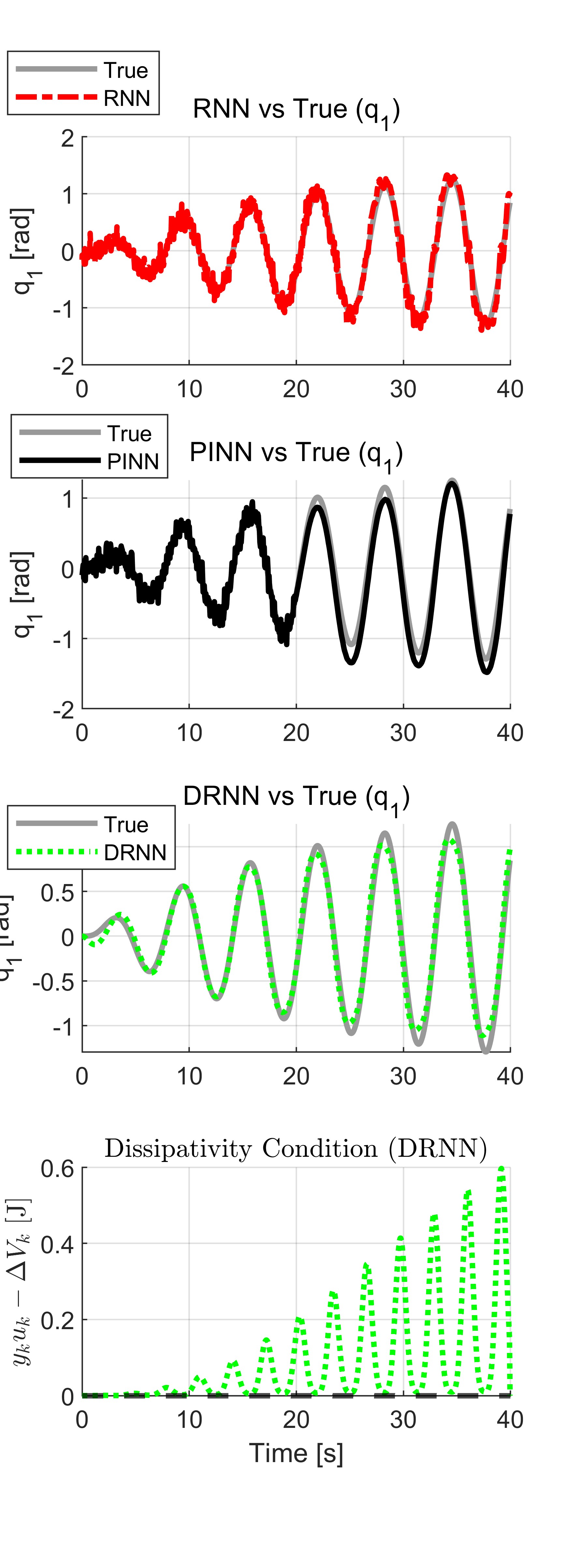}
    \caption{Comparison between RNN, DRNN, and PINN models when the training data are corrupted by noise (Gaussian noise with zero mean and standard deviation $\sigma = 0.1$). The results show that the proposed DRNN achieves superior prediction performance and exhibits lower sensitivity to noise compared with both RNN and PINN, indicating the robustness of the proposed DissipNet. Data from $0\,\mathrm{s}$ to $20\,\mathrm{s}$ are used for training, while data from $20\,\mathrm{s}$ to $40\,\mathrm{s}$ are used for testing.}
    \label{fig:drnn_rnn_2dof_noise}
\end{figure}
%
%%%%%%%%%%%%%%%%% 
Fig.~\ref{fig:drnn_pinn_msd} compares the proposed DissipNet (DRNN) with a naive RNN and a PINN on the MSD model. While all three networks can capture the system dynamics of the MSD system, notable differences arise in terms of energy behavior. In particular, the PINN does not preserve the dissipativity property, as evidenced by the corresponding energy evolution. In contrast, the energy plot of the proposed DRNN (bottom subplot) clearly demonstrates dissipative behavior, confirming that the dissipativity constraint is effectively enforced by the network architecture.
Fig.~\ref{fig:drnn_rnn_msd_noise} compares the performance of the models when the training data are corrupted by noise. The results show that the proposed DRNN achieves superior prediction accuracy with less noise sensitivity compared with both the RNN and PINN, while simultaneously preserving the dissipativity property.
Table~\ref{tab:msd_comparison} presents a quantitative comparison of the mean prediction errors and corresponding standard deviations obtained using RNN, PINN, and the proposed DRNN for the MSD model, under varying training noise conditions and numbers of training samples. In all evaluated scenarios, the proposed DRNN consistently achieves significantly lower prediction errors than the standard RNN, demonstrating the effectiveness of enforcing dissipativity constraints in improving learning stability and accuracy. For $N=200$ with training noise, DRNN reduces the mean error from $6.53\times10^{-2}$ to $1.12\times10^{-2}$, corresponding to an improvement of over $80\%$, while also exhibiting a markedly smaller variance. A similar trend is observed for $N=400$, where DRNN attains a mean error of $1.24\times10^{-2}$ compared to $6.59\times10^{-2}$ for RNN, indicating that the performance advantage of DRNN persists as the dataset size increases.

When compared with PINN, DRNN demonstrates superior robustness to training noise, particularly in the low-data regime. For $N=200$ with noise, DRNN achieves nearly a fivefold reduction in error relative to PINN, while for $N=400$ the improvement remains substantial. In the absence of training noise, both DRNN and PINN benefit from increased data availability; however, DRNN consistently attains the lowest error across all cases, reaching $6.0\times10^{-4}$ for $N=400$. These results collectively indicate that incorporating dissipativity directly into the network dissipativity directly into the network architecture leads to improved generalization and noise robustness compared with both purely data-driven RNNs and physics-informed PINNs in modeling MSD dynamics.

\begin{table}[t]
\centering
\caption{Comparing mean errors ($\pm$ std) of prediction performance using RNN, DRNN and PINN models for the MSD model (N is the number of training samples)}
\label{tab:msd_comparison}
\scriptsize
\setlength{\tabcolsep}{4pt}
\begin{tabular}{|c|c|c|c|c|}
\hline
\textbf{$N$} & \textbf{Training noise}
& \textbf{RNN ($\times 10^{-5}$)}
& \textbf{DRNN ($\times 10^{-5}$)}
& \textbf{PINN ($\times 10^{-5}$)} \\
\hline
\multirow{2}{*}{200}
& Yes
& $6.53 \pm 0.01$
& $1.12 \pm 0.00$
& $5.60 \pm 0.01$ \\
\cline{2-5}
& No
& $0.43 \pm 0.00$
& $0.07 \pm 0.00$
& $0.26 \pm 0.00$ \\
\hline
\multirow{2}{*}{400}
& Yes
& $6.59 \pm 0.00$
& $1.24 \pm 0.00$
& $4.47 \pm 0.00$ \\
\cline{2-5}
& No
& $0.30 \pm 0.00$
& $0.06 \pm 0.00$
& $0.07 \pm 0.00$ \\
\hline
\end{tabular}
\end{table}

%%
%%%
\begin{table}[t]
\centering
\caption{Comparing mean errors ($\pm$ std) of prediction performance using RNN, DRNN and PINN models for the 2DOF manipulator model (N is the number of training samples)}
\label{tab:2dof_comparison}
\scriptsize
\setlength{\tabcolsep}{4pt}
\begin{tabular}{|c|c|c|c|c|}
\hline
\textbf{$N$} & \textbf{Noise} 
& \textbf{RNN ($\times 10^{-4}$)} 
& \textbf{DRNN ($\times 10^{-4}$)} 
& \textbf{PINN ($\times 10^{-4}$)} \\
\hline
\multirow{2}{*}{200}
& Yes
& $4.60 \pm 3.49$
& $3.59 \pm 2.91$
& $1.39 \pm 1.92$ \\
\cline{2-5}
& No
& $3.60 \pm 3.28$
& $2.10 \pm 1.64$
& $0.56 \pm 1.23$ \\
\hline
\multirow{2}{*}{400}
& Yes
& $2.50 \pm 2.29$
& $1.66 \pm 1.43$
& $1.84 \pm 2.69$ \\
\cline{2-5}
& No
& $1.60 \pm 2.22$
& $0.76 \pm 0.57$
& $0.91 \pm 2.23$ \\
\hline
\end{tabular}
\end{table}
%%%
In Fig.~\ref{fig:drnn_rnn_2dof}, the proposed DRNN is compared with a naive RNN and a PINN on the modeling of a 2-DOF manipulator. Similar to the MSD system, the PINN fails to preserve the dissipativity property of the modeled system, in contrast to the proposed DRNN. Fig.~\ref{fig:drnn_rnn_2dof_noise} illustrates the performance of all models when the robot joint angles are corrupted by Gaussian noise. It is seen that DRNN is less sensitive to noise, and in both Fig.~\ref{fig:drnn_rnn_2dof} and Fig.~\ref{fig:drnn_rnn_2dof_noise}, the proposed DRNN consistently achieves lower prediction errors than both the RNN and PINN, while preserving dissipativity. This superior modeling performance is further confirmed by the quantitative results reported in Table~\ref{tab:2dof_comparison}.
Table~\ref{tab:2dof_comparison} compares the prediction error of RNN, PINN, and the proposed DRNN under different data lengths and noise conditions. Across all settings, DRNN consistently achieves lower mean error and reduced variance compared to the standard RNN, indicating improved accuracy and robustness due to the imposed dissipative structure. In particular, for $N=200$, with noise, DRNN reduces the error by approximately $22\%$ relative to RNN, while for $N=400$ the reduction increases to about $40\%$. Compared with PINN, DRNN performs competitively under noisy conditions and shows a clear advantage in noise-free cases, especially as the data length increases, where DRNN attains the lowest error among all methods. These results suggest that explicitly enforcing dissipativity leads to better generalization and noise robustness than both purely data-driven RNNs and physics-regularized PINNs. 

\section{Conclusions} \label{sec:conclusion}
In this work, we proposed a dissipative recurrent network model that can learn the dynamics and preserve the dissipative property of a known dissipative nonlinear model. This was realized through the design and learning algorithm of the network structure so that the weights satisfy dissipative constraints and they can be solved with non-constrained optimization method. The proposed network was evaluated and compared with a PINN model and a naive RNN on a mass-spring-damper system and a 2-dof planar manipulator, which are dissipative with particular input-output pairs. The results showed the superior performance of our dissipative RNN network compared to Naive RNN and PINN networks in guaranteeing both good modeling performance and physical meaning preservation. Moreover, it can be observed that, while PINN networks do not guarantee the preservation of the dissipativity property of the system, they also require prior knowledge of the system dynamics in the form of explicit equations, which is not needed in the proposed network. 

This is expected to provide a useful tool for designing a more effective and physically meaningful network applicable across a wide range of applications. Research on the modeling of the proposed model on more complex case studies, and improving the training time will be our future work. 

\appendix

\section[Proof of the dissipativity inequality]{\proofdeltaVtitle}
\label{sec:provedeltaV}
Left and right multiplying Eq.~(\ref{eq:lmi}) by $[\Nu_k^T \quad \Nu_{k+1}^T \quad \u_k^T]$ and $[\Nu_k^T \quad \Nu_{k+1}^T \quad \u_k^T]^T$ respectively, one can obtain
\begin{multline} \label{eq:qsr_expand}
    \Nu_k^T \P \Nu_k - \Nu_{k+1}^T \P \Nu_{k+1}  
    + \Nu_k^T \C_h^T \Q \C_h \Nu_k + 2\Nu_k^T \D^T \Q \C_h \u_k \\+ 2\Nu_k^T \C_h^T \S^T \u_k + \u_k^T \R_1\u_k 
     - 2 \gamma \v_k^T \LambdaB \Nu_{k+1}  + \Nu_{k+1}^T \LambdaB \Nu_{k+1} \geq 0 
\end{multline} 
taking into account that $\v_k = \W_{u} \u_k + \W_{ih} \Nu_{k+1} + \W_{hh}\Nu_{k}$, and $V_{k+1}(\Nu_{k+1})-V_{k}(\Nu_{k}) = \Nu_{k+1}^T \P\Nu_{k+1} - \Nu_{k}^T \P \Nu_{k}$, and 
\begin{multline} \label{eq:yu_expand}
    [\y_k^T \quad \u_k^T]
    \begin{bmatrix}
        \Q & \S^T \\
        \S    & \R
    \end{bmatrix}
    \begin{bmatrix}
        \y_k\\
        \u_k 
    \end{bmatrix} 
= \Nu_k^T \C_h^T \Q \C_h \Nu_k \\
+ 2\Nu_k^T \D^T \Q \C_h \u_k + 
2\Nu_k^T \C_h^T \S^T \u_k + \u_k^T \R_1 \u_k 
\end{multline}
where $\R_1 =  \R + \S \D + \D^T \S^T +  \D^T \Q \D$.

Eq.~(\ref{eq:qsr_expand}) then can be re-written as
\begin{multline}
    V_{k+1}(\Nu_{k+1})-V_{k}(\Nu_{k}) \leq 
     [\y_k^T \quad \u_k^T]
    \begin{bmatrix}
        \Q & \S^T \\
        \S    & \R
    \end{bmatrix}
    \begin{bmatrix}
        \y_k\\
        \u_k 
    \end{bmatrix}\\
    + [\v_k^T \quad \Nu_{k+1}^T]
    \begin{bmatrix}
        \0 & -\gamma \LambdaB \\
         -\gamma \LambdaB    & \LambdaB
    \end{bmatrix}
    \begin{bmatrix}
        \v_k\\
        \Nu_{k+1} 
    \end{bmatrix}
\end{multline}

\section[Proof of the slope-restriction inequality]{\proofSlopeTitle}
\label{sec:slopeproof}
The slopes of widely-used activation functions $\phi()$, such as ReLU, tanh, exponential linear functions, are restricted by the range [alpha, beta]. We  have
\begin{equation}
    \alpha \v_k \leq \Nu_{k+1} = \phi(\v_k)-\phi(\0)\leq \beta \Nu_k
\end{equation}
Therefore
\begin{equation}
    (\Nu_{k+1}-\alpha \v_k)\LambdaB ( \Nu_{k+1} -\beta \v_k)
\end{equation}

\begin{equation} \label{eq:slope}
    [\v_k^T \quad \Nu_{k+1}^T]
    \begin{bmatrix}
        \rho \LambdaB & -\gamma \LambdaB \\
         -\gamma \LambdaB    & \LambdaB
    \end{bmatrix}
    \begin{bmatrix}
        \v_k\\
        \Nu_{k+1}
    \end{bmatrix} \leq 0
\end{equation}
where, $\rho = \alpha \beta, \gamma = \frac{1}{2}(\alpha + \beta)$. From assumption~\ref{assump1}, Eq.~(\ref{eq:slope}) becomes
\begin{equation} \label{eq:slope_restricted}
    [\v_k^T \quad \Nu_{k+1}^T]
    \begin{bmatrix}
        \0 & -\gamma \LambdaB \\
         -\gamma \LambdaB    & \LambdaB
    \end{bmatrix}
    \begin{bmatrix}
        \v_k\\
        \Nu_{k+1}
    \end{bmatrix} \leq 0
\end{equation}
\subsection{Choose \texorpdfstring{$\D$}{D} such that} \label{sec:D}
\begin{equation} \label{eq:R}
    \R_1 = \R + \S \D + \D^T \S^T + \D^T \Q \D \succ 0
\end{equation}
It is noticed that Eq.~(\ref{eq:R}) will be satisfied if
\begin{equation} \label{eq:D_cond}
\L = \R + \S \D + \D^T \S^T + \D^T \Q_1 \D \succ 0 
\end{equation}
where $\Q_1 = \Q - \epsilon \I $ with $\I \in \mathbb{R}^{p \times p}$, and $\epsilon$ is a sufficiently small constant to guarantee that $\Q_1 \prec 0$. $\epsilon$ can be chosen as 0 in case $\Q \prec \0$. Then $\Q_1$ can be represented as $\Q_1 = -\L_q^T \L_q$.

$\D$ can be chosen as follow~\cite{revay2023recurrent}
\begin{equation} \label{eq:D_cond3}
    \D = -\Q_1^{-1} \S^T + \L_q^{-1} \N \L_r
\end{equation}
where $\L_r\in \mathbb{R}^{m \times m}$ is found from Eq.~(\ref{eq:Lr})
\begin{equation} \label{eq:Lr}
    \R -\S \Q_1 \S^{T} = \L_r^T \L_r
\end{equation}
%%%
$\N$ and $\M$ can be chosen as follows

if $p \geq m$
\begin{align}
\label{eq:MN}
    \M=\Y_1^T\Y_1 + \Y_2 - \Y_2^T + \Y_3^T\Y_3 + \epsilon_3 \I_m \\
    \N = \begin{bmatrix}
        (\I_m-\M)(\I_m+\M)^{-1}\\
        -2\Y_3(\I_m + \M)^{-1}
    \end{bmatrix}
\end{align} 
where $\Y_1\in \mathbb{R}^{m \times m}, \Y_2\in \mathbb{R}^{m \times m}$, and $\Y_3\in \mathbb{R}^{(p-m) \times m}$ are free variables, $\epsilon_3 > 0$, $\I_m$ is the identity matrix with size $ m \times m$.

if $p < m$
\begin{align}
\label{eq:MN_second_case}
    \M=\Y_1^T\Y_1 + \Y_2 - \Y_2^T + \Y_3^T\Y_3 + \epsilon_3 \I_p \\
    \N = \begin{bmatrix}
        (\I_m+\M)^{-1}(\I_m-\M)  \quad      -2(\I_m + \M)^{-1}\Y_3^T
    \end{bmatrix}
\end{align} 
where $\Y_1\in \mathbb{R}^{p \times p}, \Y_2\in \mathbb{R}^{p \times p}$, and $\Y_3\in \mathbb{R}^{(m-p) \times p}$ are free variables, $\I_p$ is the identity matrix with size $ p \times p$.

It can be proven that $\N\N^T \prec \I$ and $\R_1 = \L_r^T(\I - \N\N^T)\L_r \succ 0$. where $\I$ the identity matrix with the same size as $\N\N^T$.

\section*{Acknowledgment}
This material is based upon work supported by the Air Force Office of Scientific Research under award number FA2386-24-1-4039.

\bibliographystyle{ieeetr}
\bibliography{reference}

\end{document}